\documentclass[10pt,twocolumn]{IEEEtran}

\usepackage{booktabs}
\usepackage{subfigure}
\usepackage{amssymb}
\usepackage{amsfonts}
\usepackage{amsmath}
\usepackage{array}
\usepackage{threeparttable}
\usepackage{booktabs}
\usepackage{cite}
\usepackage[dvips]{graphicx}
\usepackage{color}
\usepackage{comment}
\usepackage{makecell}
\usepackage{multirow}
\usepackage{textcomp,mathcomp}
\usepackage[linesnumbered, ruled, vlined]{algorithm2e}
\usepackage{subcaption}
\usepackage{caption}
\usepackage{pifont}
\usepackage[colorlinks=true, linkcolor=red]{hyperref}

\begin{document}
\newtheorem{definition}{\it Definition}
\newtheorem{theorem}{\bf Theorem}
\newtheorem{lemma}{\it Lemma}
\newtheorem{corollary}{\bf Corollary}
\newtheorem{remark}{\it Remark}
\newtheorem{example}{\it Example}
\newtheorem{case}{\bf Case Study}
\newtheorem{assumption}{\it Assumption}
\newtheorem{property}{\it Property}
\newtheorem{proposition}{\bf Proposition}

\newcommand{\hP}[1]{{\boldsymbol h}_{{#1}{\bullet}}}
\newcommand{\hS}[1]{{\boldsymbol h}_{{\bullet}{#1}}}

\newcommand{\ba}{\boldsymbol{a}}
\newcommand{\baq}{\overline{q}}
\newcommand{\bA}{\boldsymbol{A}}
\newcommand{\bb}{\boldsymbol{b}}
\newcommand{\bB}{\boldsymbol{B}}
\newcommand{\bc}{\boldsymbol{c}}
\newcommand{\bC}{\boldsymbol{C}}
\newcommand{\bd}{\boldsymbol{d}}
\newcommand{\bcD}{\boldsymbol{\cal D}}
\newcommand{\bcO}{\boldsymbol{\cal O}}
\newcommand{\bh}{\boldsymbol{h}}
\newcommand{\bH}{\boldsymbol{H}}
\newcommand{\bl}{\boldsymbol{l}}
\newcommand{\bm}{\boldsymbol{m}}
\newcommand{\bn}{\boldsymbol{n}}
\newcommand{\bo}{\boldsymbol{o}}
\newcommand{\bO}{\boldsymbol{O}}
\newcommand{\bp}{\boldsymbol{p}}
\newcommand{\bq}{\boldsymbol{q}}
\newcommand{\bR}{\boldsymbol{R}}
\newcommand{\bs}{\boldsymbol{s}}
\newcommand{\bS}{\boldsymbol{S}}
\newcommand{\bT}{\boldsymbol{T}}
\newcommand{\bw}{\boldsymbol{w}}
\newcommand{\bz}{\boldsymbol{z}}

\newcommand{\balpha}{\boldsymbol{\alpha}}
\newcommand{\bbeta}{\boldsymbol{\beta}}
\newcommand{\bgamma}{\boldsymbol{\gamma}}
\newcommand{\bkappa}{\boldsymbol{\kappa}}
\newcommand{\bomega}{\boldsymbol{\omega}}
\newcommand{\btomega}{\boldsymbol{\tilde \omega}}

\newcommand{\bOmega}{\boldsymbol{\Omega}}
\newcommand{\bGamma}{\boldsymbol{\Gamma}}

\newcommand{\bTheta}{\boldsymbol{\Theta}}
\newcommand{\bphi}{\boldsymbol{\phi}}
\newcommand{\btheta}{\boldsymbol{\theta}}
\newcommand{\bvarpi}{\boldsymbol{\varpi}}
\newcommand{\bpi}{\boldsymbol{\pi}}
\newcommand{\bpsi}{\boldsymbol{\psi}}
\newcommand{\brho}{\boldsymbol{\rho}}
\newcommand{\bxi}{\boldsymbol{\xi}}
\newcommand{\bx}{\boldsymbol{x}}
\newcommand{\by}{\boldsymbol{y}}

\newcommand{\cA}{{\cal A}}
\newcommand{\bcA}{\boldsymbol{\cal A}}
\newcommand{\cB}{{\cal B}}
\newcommand{\cD}{{\cal D}}
\newcommand{\cE}{{\cal E}}
\newcommand{\cG}{{\cal G}}
\newcommand{\cH}{{\cal H}}
\newcommand{\cI}{{\cal I}}
\newcommand{\bcH}{\boldsymbol {\cal H}}
\newcommand{\cJ}{{\cal J}}
\newcommand{\cK}{{\cal K}}
\newcommand{\cL}{{\cal L}}
\newcommand{\cM}{{\cal M}}
\newcommand{\cO}{{\cal O}}
\newcommand{\cR}{{\cal R}}
\newcommand{\cS}{{\cal S}}
\newcommand{\dcS}{\ddot{{\cal S}}}
\newcommand{\ds}{\ddot{{s}}}
\newcommand{\cT}{{\cal T}}
\newcommand{\cU}{{\cal U}}
\newcommand{\wt}[1]{\widetilde{#1}}

\newcommand{\mA}{\mathbb{A}}
\newcommand{\mE}{\mathbb{E}}
\newcommand{\mG}{\mathbb{G}}
\newcommand{\mR}{\mathbb{R}}
\newcommand{\mS}{\mathbb{S}}
\newcommand{\mU}{\mathbb{U}}
\newcommand{\mV}{\mathbb{V}}
\newcommand{\mW}{\mathbb{W}}

\newcommand{\uq}{\underline{q}}
\newcommand{\ubq}{\underline{\boldsymbol q}}

\newcommand{\red}[1]{\textcolor[rgb]{1,0,0}{#1}}
\newcommand{\gre}[1]{\textcolor[rgb]{0,1,0}{#1}}
\newcommand{\blu}[1]{\textcolor[rgb]{0,0,0}{#1}}
\newcommand{\ltgr}[1]{\textcolor[rgb]{0.6,0.6,0.6}{#1}}

\newcommand{\best}{$\uparrow$}
\newcommand{\worst}{$\downarrow$}

\title{Generalization Bounds of Emergent Communications for Agentic AI Networking}

\author{\IEEEauthorblockA{Yong Xiao\IEEEauthorrefmark{1}\IEEEauthorrefmark{2}\IEEEauthorrefmark{3}, Jingxuan Chai\IEEEauthorrefmark{4}, Guangming Shi\IEEEauthorrefmark{4}\IEEEauthorrefmark{2}, Ping Zhang\IEEEauthorrefmark{5} \\ 
\IEEEauthorblockA{\IEEEauthorrefmark{1} School of Elect. Inform. \& Commun., Huazhong Univ. of Science \& Technology, China}\\
\IEEEauthorblockA{\IEEEauthorrefmark{2} Peng Cheng Laboratory, Shenzhen, China}\\
\IEEEauthorblockA{\IEEEauthorrefmark{3} Pazhou Laboratory (Huangpu), Guangzhou, China}\\
\IEEEauthorblockA{\IEEEauthorrefmark{4} School of Artificial Intelligence, Xidian University, Xi'an, China}\\
\IEEEauthorblockA{\IEEEauthorrefmark{5} State Key Lab. of Networking \& Switching Tech., Beijing Univ. of Posts \& Telecom., Beijing, China}
\thanks{*This work will be presented at IEEE ISIT Workshop, Guangzhou, China, June 2026. Copyright may be transferred without notice, after which this version may no longer be accessible.} 
}
}

\maketitle

\begin{abstract}
The evolution of 6G networking toward agentic AI networking (AgentNet) systems requires a shift from traditional data pipelines to task-aware, agentic AI-native communication solutions. Emergent communication, a novel communication paradigm in which autonomous agents learn their own signaling protocols through interaction, is increasingly viewed as a promising solution to address the challenges posed by existing rigid, predefined protocol-based networking architecture. However, most existing emergent communication frameworks fail to account for physical networking constraints, such as bandwidth and computational complexity, and often lack a rigorous information-theoretical foundation. To address these challenges, this paper introduces a novel emergent communication framework that facilitates collaborative task-solving among heterogeneous agents through an information-theoretic lens. We propose a novel joint loss function that unifies the optimization of decision-making functions and the learning of communication signaling. Our proposed solution is grounded on the multi-agent and multi-task distributed information bottleneck (DIB) theory, which allows the quantification of the fundamental trade-off between task-relevant information representation and computational complexity. We further provide theoretical generalization bounds of the emergent communication protocol during decentralized inference across unseen environmental states. 
Experimental validation on a real-world hardware prototype confirms that our proposed framework significantly improves generalization performance, compared to the state-of-the-art solutions. 
\end{abstract}
\begin{IEEEkeywords}
Emergent communication, agentic AI networking, generalization bounds, 6G. 
\end{IEEEkeywords}

\section{Introduction}
\label{Section_Introduction}

The rapid evolution of mobile networking systems toward Agentic AI networking (AgentNet) introduces novel challenges for traditional communication frameworks and networking protocols\cite{xiao2025AgentNet,XY2024ACMGetMobileSAN,xiao2023reasoning,shi2020semantic,ZhangPing2022EngineeringSemanticComm}. In particular, existing communication systems rely on rigid, predefined rules and fixed frame structures that struggle to accommodate the dynamic, multi-modal interactions among diverse AI agents. These protocols often lack the flexibility and scalability required for real-time coordination among heterogeneous AI agents. Because traditional protocols are task-agnostic, predefined, and fixed, they cannot adapt to specific task needs, creating a significant barrier to efficient task-driven multi-agent collaboration. In many cases, the rigid constraints of human-designed protocols lead to either excessive signaling overhead or a lack of awareness of the critical task-relevant context.

Emergent communication offers a promising solution to these challenges by enabling collaborative agents to develop their own communication protocols through direct interaction \cite{XY2026SANEmerg}. Rather than adhering to fixed, hand-crafted rules, agents learn to exchange signals that maximize their collective utility, discovering adaptive messaging schemes tailored to their specific task requirements and environmental features. This process enables the creation of grounded, semantically rich interfaces that evolve autonomously alongside agents' capabilities. By treating communication as a learned action, the network can autonomously optimize, learn, and evolve, ensuring that emergent communication signaling and protocols are perfectly aligned with the agents' decision-making processes\cite{Boldt2025EmergCommSearching,Chen2026FiveWofMultiAgentComm}.


Despite its promising potential, most existing work on emergent communication prioritizes the emergence of human-language-like features while largely overlooking the practical constraints of physical networking systems, such as bandwidth limitations and the computational complexity limits of agents. Furthermore, most existing works remain predominantly experimental and heuristic-driven; most frameworks lack the rigorous information-theoretic foundations necessary to guarantee performance and robustness in real-world, resource-constrained networking environments\cite{XY2026SANEmerg}. 

This motivates the work in this paper, in which we focus on developing an information-theoretically grounded framework for emergent communication solutions for communication and computationally bounded AgentNet systems. In particular, we introduce a novel joint loss function that enables the simultaneous training of task-specific decision-making functions and emergent communication models. By unifying these typically disparate processes into a single optimization objective, our proposed solution eliminates the computational and communicational redundancy inherent in developing separate modules for task execution and communication signaling. We also establish a rigorous theoretical foundation for multi-agent emergent communication based on Distributed Information Bottleneck (DIB) theory. This framework provides the mathematical tools necessary to quantify and optimize the fundamental trade-off between preserving task-essential information (relevance) and minimizing the computational complexity, quantified by the Minimum Description Length (MDL), of the emergent signals during communication. We present a theoretical analysis of the generalization capabilities of the learned communication protocols. By deriving bounds on the generalization error, we provide mathematical guarantees on the performance consistency and structural stability of emergent communication when transitioning from training environments to unseen decentralized inference scenarios. Finally, the proposed emergent communication framework is validated through extensive experiments conducted on a hardware mobile networking prototype. The empirical results demonstrate that our proposed solution outperforms state-of-the-art emergent communication solutions, achieving substantial improvements in generalization performance and robustness under dynamic networking environments.

\section{System Model and Problem Formulation}
\label{Section_SystemModel}

\begin{figure}
\centering
\includegraphics[width=0.8\linewidth]{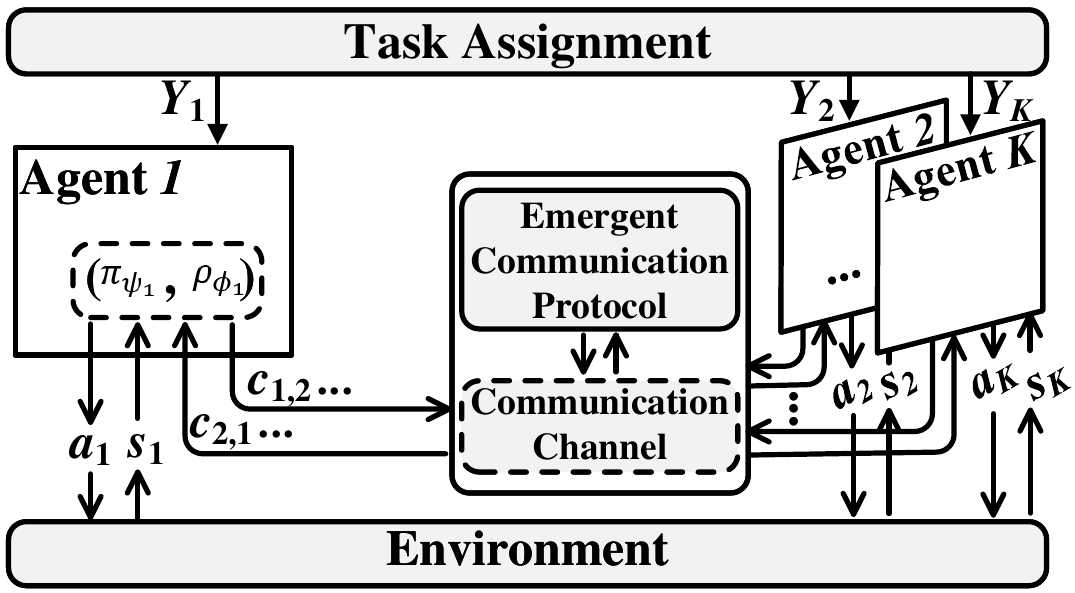}
\caption{\footnotesize A multi-agent emergent communication system model.}
\label{Figure_EmergCommModel}
\vspace{-0.1in}
\end{figure}



We consider a multi-agent emergent communication system in which a set of agents is selected to collaborate on solving a set of tasks, defined as a tuple $\langle {\cal K}, {\cal Y}, \boldsymbol{\cal S}, \bc, \bpi \rangle$, illustrated in Fig. \ref{Figure_EmergCommModel}.  

\noindent{\bf (1) A set of agents} ${\cal K}$: A set of $K$ agents, $\mathcal{K} = \{1, 2, \dots, K\}$, is deployed in a shared physical or networking environment. Generally speaking, different agents observe different states of the environment and have different specialized functional capabilities. For example, a physical-layer agent may observe channel state information (CSI) and possess skillsets in channel adaptation and waveform optimization. In contrast, an application-layer agent may observe diverse modalities of application data sequences, such as real-time video streams or semantic metadata, and can specialize in intent recognition or application-level parameter tuning. These agents can be selected to jointly solve cross-layer objectives: the application-layer agent identifies and extracts the most task-relevant semantics and adapts source parameters, e.g., video streaming resolution, to meet user requirements, while the physical-layer agent simultaneously optimizes transmission strategies to ensure these critical features are prioritized over a fluctuating wireless channel\cite{XY2026SANet, Zhu2024SANSee}.

\noindent{\bf (2) Task space} ${\cal Y}$: A set of tasks or intents, each of which has a specific optimization objective, i.e., a loss function, is assigned to the agents. We assume each agent $k$ is assigned a specific task, and we can therefore use $\cL_k$ to denote the associated loss function assigned to agent $k$. 
Generally speaking, each agent $k$'s local observation is insufficient to perfectly solve the assigned task $Y_k$. Thus, agent $k$ may require supplementary information from other agents. To facilitate the collaborative task solving and communication emergence during the training phase, for each specific task, a task-specific random variable $Y_k \in {\cal Y}$, e.g., an intended reward signal, is provided to each agent $k$ in the training dataset to guide its policy optimization. This variable represents the objective feedback or ``ground truth" objectives that the agent aims to optimize, essentially serving as the semantic anchor for both its local decision-making and its communication protocol development. Consider the cross-layer optimization scenario as an example. During the model training phase, whenever a user's Quality-of-Experience (QoE) requirement is satisfied, e.g., by maintaining a high-definition video stream with negligible jitter, a positive reward instance $y_k$ is issued to both the application-layer agent and the physical-layer agent to encourage the application-layer agent to refine its semantic extraction and source adaptation policies while simultaneously steering the physical-layer agent toward more robust channel adaptation and transmit parameter optimization strategies. In this case, both agents learn to develop an emergent protocol that prioritizes the exchange of information most critical to sustaining user satisfaction across the various layers of the networking systems. Note that this reward variable can be provided offline as a dedicated training dataset prior to deployment, or delivered online during inference to facilitate real-time model updates; in either case, the agents continue to collaborate in a fully decentralized manner during inference.

    
\noindent{\bf (3) State space} $\boldsymbol{\cal S}$: Each agent can have a local, noisy, and partial observation of the globally shared environmental state $S_k$. Let ${\cal S}_k$ be the local state space associated with agent $k$, i.e., we have $S_k \in {\cal S}_k$. The joint observations of all the agents $\langle S_1, \dots, S_K \rangle$ can be highly correlated, since they can all be considered partial observations or projections of the shared environmental state. 
    
\noindent{\bf (4) Emergent communication signaling} $\bC$: The agents can learn emergent communication signals or protocols to facilitate their coordination and collaborative optimization. In this paper, we consider the deep-learning-based emergent communication\cite{Lazaridou2020EmergMultiAgentComm,XY2026SANEmerg}, in which each agent $k$ learns a communication emergent function $\rho_{\phi_k}: s_k \rightarrow \bc_{k,-k}$ parameterized by $\phi_k$ that maps its locally observed state to a set of messages $\bc_{k,-k} = \langle C_{k,j} \rangle_{j\in {\cal K}\backslash k}$ sent to other agents, where $s_k$ is a state realization of agent $k$, i.e., the realized message vector sent by agent $k$ to other agents $\bc_{k,-k} = \langle c_{k,j} \rangle_{j \in \mathcal{K} \setminus k}$ is given by 
    \begin{eqnarray}
    \bc_{k,-k} \sim \rho_{\phi_k}( \bC_{k,-k} | S_k = s_k ).
    \end{eqnarray}
   where $\bC_{k,-k} = \langle C_{k,j} \rangle_{j \in \mathcal{K} \setminus k}$ is the random vector of messages generated by agent $k$.
%
%
    
\noindent{\bf (5) Decision making functions} $\bpi$: Each agent $k$ learns a decision making function to decide an action $A_k \in \mathcal{A}_k$ to solve its assigned task. The realization of this action $a_k$ is decided based on the aggregation of its locally observed state realization $s_k$ and the vector of realized messages received from other agents $\bc_{-k, k} = \langle c_{j,k} \rangle_{j \in \mathcal{K} \setminus k}$. We again consider a deep-learning-based learning solution in which the above decision-making process is governed by a mapping function parameterized by $\psi_k$, i.e., we have
    \begin{equation}
    a_k \sim \pi_{\psi_k}(A_k | S_k, \bC_{-k,k} = \bc_{-k, k}).
    \end{equation}

The main objective of the above multi-agent communication framework is to autonomously learn an emergent communication protocol that extracts the most task-relevant information from localized observations, while simultaneously minimizing both communication overhead and computational complexity at the agents.  
We can therefore formulate the above problem as a decentralized multi-agent multi-task problem, in which each agent focuses on solving its local task by minimizing the task-specific loss, while simultaneously extracting the most useful feature signals to help other agents solve their tasks. We can formulate the joint optimization problem as follows:
\begin{equation}
{\bf P1:} \; \min_{\langle \bpi_{\bpsi}, \brho_{\bphi} \rangle } \sum_{k\in {\cal K}} \cL_k (\pi_{\psi_k}, \rho_{\phi_k}; s_{k}, \bc_{-k,k}),
\end{equation}
where $\bpi_{\bpsi} = \langle \pi_{\psi_k} \rangle_{k\in {\cal K}}$ and $\brho_{\bphi} = \langle \rho_{\phi_k} \rangle_{k\in {\cal K}}$. 
We can observe that optimizing the above decentralized multi-agent communication systems across diverse, simultaneous tasks and agents presents a formidable theoretical and computational challenge. In particular, since agents collaboratively address a diverse set of tasks based on partial environmental observations, their signaling strategies and local decision-making processes become deeply entangled, creating a complex, non-stationary learning landscape in which each agent's optimal policy continuously shifts in response to evolving communication protocols and other agents' actions. 
%
This often leads to significant coordination overhead and the risk of suboptimal convergence, as agents may struggle to distinguish task-relevant features from environmental noise across their correlated yet distinct observations. Given these challenges, a direct optimization of problem {\bf P1} is generally computationally intractable and theoretically impossible. 


\section{A DIB-based Framework}

To address the above challenges, we propose a novel emergent communication framework grounded in the Distributed Information Bottleneck (DIB) theory. In this way, we can transform the intractable joint optimization problem {\bf P1} into an information-theoretic problem focused on extracting the most task-relevant representation from agents' local observations while minimizing communication overhead. The core idea of DIB is to formalize representation learning as a rigorous mathematical trade-off between informational representation complexity and task-relevant information representation. 
%
%
%
More specifically, for each agent $k$, the above trade-off is governed by two competing mutual information metrics: 

\noindent{\bf (1) Representation (computational) complexity} $I(S_k; \bC_{k,-k})$: This term measures how much raw information from the local observation $S_k$ is retained in the emergent message $C_k$. This term is, in fact, the Minimum Description Length (MDL), a widely adopted complexity metric for assessing the computational complexity of computationally bounded intelligence systems\cite{finzi2026epiplexity, rdc2026chai}.

\noindent{\bf (2) Task-relevant (communication) information} $I(Y_k; \bC_{-k,k})$: This term captures the amount of information the emergent message $C_k$ provides regarding the task-specific target intent $Y_k$.

By simultaneously minimizing complexity and maximizing task-relevant information, the collaborative set of agents can learn to transmit only the highly distilled, non-redundant ``state differences" information necessary to solve the task assigned to each agent, while filtering out the task-irrelevant state noise. More formally, we can reformulate the multi-agent emergent communication problem as follows:
\begin{eqnarray}
\label{eq:opt_p2}
&{\bf P2:}& \; \min_{\langle \bpi_{\bpsi}, \brho_{\bphi} \rangle } \sum_{k\in {\cal K}} \cL_k (\pi_{\psi_k},  \rho_{\phi_k}; s_{k}, \bc_{-k,k}) \\ 
&&\;\;\;\;\;\;\;\;\;\;\;\;\;\; - \lambda_k^t I(Y_k; \bC_{k,-k}) + \lambda_k^c I(S_k; \bC_{k,-k}). \nonumber 
\end{eqnarray}
where $\lambda_k^t$ and $\lambda_k^c$ are Lagrangian coefficients that control the tradeoff between complexity and task-relevant information. 


It can be observed that the true marginal distributions and mutual information terms in problem {\bf P2} are computationally intractable in high-dimensional continuous spaces. We therefore need to derive tractable variational bounds for optimizing the joint optimization objective. More specifically, we can prove the following variational bounds. 

\begin{theorem}[Variational Upper Bound of Complexity]
\label{th:complexity}
We can prove that the representation complexity for each agent $k$ has the following upper bound: 
\begin{align}
    &I(S_k; \bC_{k,-k}) \label{eq_VariationalUpperBound} \\
   \leq& \mathbb E_{s_k\sim p_{S_k}}[D_{KL}( p_{\phi_k}(\bC_{k,-k}|S_k=s_k) \Vert r^c_{\phi_k}(\bC_{k,-k}))]\nonumber
\end{align}
where $r^c_{\phi_k}$ is any distribution.
\end{theorem}
\begin{IEEEproof}
According to the definition, we have 
\begin{eqnarray}
I(S_k; \bC_{k,-k})=\mathbb E_{s_k\sim p_{S_k}}[D_{KL}(p_{\bC_{k,-k}|S_k=s_k}\Vert p_{\bC_{k,-k}})]. 
\end{eqnarray}
We can then decompose the KL divergence term $D_{KL}( p_{\phi_k}(\bC_{k,-k}|S_k=s_k) \Vert r^c_{\phi_k}(\bC_{k,-k})) = D_{KL}(p_{\bC_{k,-k}|S_k=s_k}\Vert p_{\bC_{k,-k}}) + D_{KL}(p_{\bC_{k,-k}}\Vert r^c_{\phi_k}(\bC_{k,-k}))$. By applying $D_{KL}(p_{\bC_{k,-k}}\Vert r^c_{\phi_k}(\bC_{k,-k})) \ge 0$, we can get the result in (\ref{eq_VariationalUpperBound}). 
This concludes the proof. 
\end{IEEEproof}

\begin{theorem}[Variational Lower Bound of Task-relevant Information]
\label{th:task}
We can prove that the task-relevant information term has the following lower bound:
\begin{eqnarray}
I(Y_k; \bC_{-k,k}) \geq \mathbb E_{\bc\sim p_{\bC_{-k,k}}}[\log r^t_{\phi_k}(Y_k|\bC_{-k,k}=\bc)] \label{eq_VariationalLowerBound}
\end{eqnarray}
where $r^t_{\phi_k}$ is 
the variational approximation of the true posterior. 
\end{theorem}
\begin{IEEEproof}
Following the definition, we have $I(Y_k; \bC_{-k,k}) = H(Y_k) - H(Y_k | \bC_{-k,k})$. Then, by applying the variational property $H(Y_k | \bC_{-k,k}) \le \min_{r^t_{\phi_k}} [- \log r^t_{\phi_k}(Y_k|\bC_{-k,k}=\bc) ]$, we can obtain
the result in (\ref{eq_VariationalLowerBound}). 
This concludes the proof.
\end{IEEEproof}

By substituting the above bounds into problem {\bf P2}, we can rewrite the optimization problem as follows:
\begin{eqnarray}
\lefteqn{{\bf P3:} \; \min_{\langle \bpi_{\bpsi}, \brho_{\bphi} \rangle } \sum_{k\in {\cal K}} \cL^{\bf P3}_k (\pi_{\psi_k},  \rho_{\phi_k}; s_{k}, \bc_{-k,k})} \\
&=& \sum_{k\in {\cal K}} \left(  \cL_k (\pi_{\psi_k},  \rho_{\phi_k}; s_{k}, \bc_{-k,k}) \right. \nonumber\\
&&- \lambda_k^t \mathbb E_{Y_k\bC_{-k,k}}[\log r^t_{\phi_k}(Y_k|\bC_{-k,k})]\nonumber\\
&& \left. + \lambda_k^c \mathbb E_{S_k}[D_{KL}(p_{\phi_k}(\bC_{k,-k}|S_k)\Vert r^c_{\phi_k}(\bC_{k,-k}))] \right).\nonumber
\end{eqnarray}

The joint optimization loss in problem {\bf P3} offers several advantages over traditional multi-agent emergent communication solutions. In particular, by integrating the multi-agent multi-task DIB objective into the task loss, the proposed loss effectively addresses the inherent non-stationarity and high variance found in decentralized multi-agent networking environments. While traditional solutions often suffer from ``representation drift", where the meaning of an emergent message $\bC_{k,-k}$ shifts unpredictably as other agents'  decision making functions $\bpi_{-k}$ evolve during the training, our formulation introduces a stationary semantic anchor through the local task-relevance term $I(Y_k; \bC_{-k,k})$. Furthermore, incorporating the MDL-based complexity term prevents ``informational collapse" and overfitting to environmental state noise, ensuring that the developed protocols are not only bandwidth-efficient but also robust. To summarize, the joint optimization problem {\bf P3} converts the entangled, ill-posed decentralized optimization problem {\bf P1} into a well-conditioned learning task, significantly accelerating convergence and yielding a more stable, semantically grounded communication convention for multi-agent collaboration.

\section{Generalization Bounds}
\label{Section_Theory}


To evaluate the robustness of the emergent communication protocols learned from our proposed joint optimization problem {\bf P3}, we analyze the generalization gap between the empirical loss and the expected risk. In the context of multi-agent emergent communication, a communication protocol generalizes well if the message $\bC$ derived from a finite training dataset remains a sufficient statistic for each task $Y_k$ assigned to each agent $k$ when encountering previously unseen environmental states $s_k \in \mathcal{S}_k$. 

Before presenting the specific bounds, let us first define the generalization errors. Suppose each agent has been given a training dataset ${\cal V}_k = \langle z_{k,i} \rangle_{i \in \{1, \ldots, n\}}$ where $z_{k,i} = \langle s_{k,i}, y_{k,i} \rangle$ is the $i$th data sample. Suppose the ground truth distribution of $z_{k,i}$ is given by ${\cal Z}_k$. It is important to note that the state-task pairs in the training dataset are sufficient to train both the decision-making functions and the emergent communication signaling models. In this multi-agent emergent communication problem, the training dataset provides the desired outcomes, i.e., target intents, that the agents aim to achieve. This allows the agents to learn to imitate these goals and autonomously generate communication protocols that facilitate their attainment, and the resulting actions and communication signaling are the latent variables that the agents will eventually learn to produce in order to reach the desired outcome $y_k$. This formulation differs from traditional trajectory-based learning, which would otherwise require a dataset comprising comprehensive behavioral trajectories, including sequences of historical actions and previous communication messages, to model the step-by-step evolution of the agents' policies. Therefore, in the rest of this section, we can rewrite the sample-wise joint loss function in the problem {\bf P3} as $\cL_{k,i} (\pi_{\psi_k},  \rho_{\phi_k}; z_{k,i})$ for sample $z_{k,i}$.


The generalization error of a specific agent $k$ is defined as the absolute difference between the population loss, i.e., the expected loss calculated based on the true distribution of state-task pairs given by ${\cal L}^{\rm pop}_k (\pi_{\psi_k},  \rho_{\phi_k})  $=$ \mathbb{E}_{z_k \sim {\cal Z}_k} {\cal L}_{k} (\pi_{\psi_k},  \rho_{\phi_k}; z_k)$, and empirical loss, i.e., the sample average loss calculated based on the given training dataset given by ${\cal L}^{\rm emp}_k (\pi_{\psi_k},  \rho_{\phi_k}; {\cal Z}_{k})$ $=$ ${1 \over n} \sum_{z_{k,i} \in {\cal Z}_k} {\cal L}_{k} (\pi_{\psi_k},  \rho_{\phi_k}; z_{k,i})$. 
Formally, define $w_k:=(\psi_k,\phi_k)$, the generalization error for a specific training instance $z_{k,i}$ of agent $k$ is the absolute difference between the population and empirical loss:
\begin{equation}
    \mathcal{E}^G_k (\pi_{\psi_k}, \rho_{\phi_k}) = \mathbb E_{w_k\sim p_{W|\mathcal Z_k}}| {\cal L}^{\rm pop}_k (w_k) - {\cal L}^{\rm emp}_k (w_k; {\cal Z}_{k}) |
\end{equation}
%
%
We can prove the following theoretical results: 

\begin{theorem}[Upper Bound of Generalization Error]
\label{th:gbound_point}
Assume the loss function is $\sigma$-sub-Gaussian.
Then, with probability at least $1 - \delta$,
the generalization error for the $k$-th agent with parameter $w_k:=(\psi_k,\phi_k)$ is bounded by
\begin{eqnarray}
\mathcal{E}^G_k (\pi_{\psi_k}, \rho_{\phi_k}) \le  \left( 
\frac{D_k\cdot \mathcal{M}_{k}}{\delta} \right)^{\frac{t-1}{t}},
\label{eq_Upperbound_GenError}
\end{eqnarray}
where 
\begin{eqnarray}
    D_k &=&\exp \left( \frac{t-1}{t} D_t (p_{W|\mathcal Z_k} (w) \Vert \mu (w)) \right), \\ 
    \mathcal{M}_{k} &=& 2\left(\frac{\sigma^2 t}{n(t-1)}\right)^{\frac{t}{2t-2}},
\end{eqnarray}
and $D_t$ denotes the R\'enyi divergence of order $t>0$,  $\mu$ is a given prior distribution of $w_k$.
\end{theorem}

\begin{IEEEproof}
By applying H\"older's inequality and set $p=t$ and $q=\frac{t}{t-1}$, we have 
\begin{eqnarray}
\lefteqn{\mathcal{E}_k^G (\pi_{\psi_k}, \rho_{\phi_k}) = \int R_k(w) p_{W|\mathcal Z_k}(w) dw \le }  \nonumber \\
&& \underbrace{\left( \int R_k(w)^q \mu(w) dw \right)^{1 \over q}}_{\text{Empirical Risk Term}}
\underbrace{\left( \int \Bigl( \frac{p_{W|\mathcal Z_k}(w)}{\mu(w)} \Bigr)^t \mu(w) dw \right)^{1 \over t}}_{\text{Complexity Term}} \nonumber
\end{eqnarray}
where, for the empirical risk term, we can apply the $\sigma$-Sub-Gaussian moment bound\cite{boucheron2013concentration} and obtain the following result: 
\begin{eqnarray}
\left( \int R_k(w)^q \mu(w) dw \right)^{1/q}\leq 
    \left( \frac{\mathcal{M}_{k}}{\delta} \right)^{1/q},
    \label{eq:moment_bound}
\end{eqnarray}
and the complexity term corresponds to the exponential of the R\'enyi divergence of order $t$, given by
\begin{eqnarray}
\lefteqn{ \left( \int \Bigl( \frac{p_{W|\mathcal Z_k}(w)}{\mu(w)} \Bigr)^t \mu(w) dw \right)^{1/t} =} \nonumber \\ 
&& \exp \left( \frac{t-1}{t} D_t (p_{W|\mathcal Z_k}(w) \Vert \mu(w)) \right) = D_k^{\frac{t-1}{t}}. 
\label{eq:complex}
\end{eqnarray}

By combining (\ref{eq:moment_bound}) and (\ref{eq:complex}), we can obtain 
(\ref{eq_Upperbound_GenError}).
\end{IEEEproof} 

\begin{remark}
In Theorem~\ref{th:gbound_point}, the term $D_k$, which measures the divergence between the learned posterior $p_{W|\mathcal Z_k}$ and the prior $\mu$, quantifies the complexity of learning an emergent protocol. A smaller $D_k$ implies that the agent's emergent protocol does not over-specialize to specific training datasets. This aligns with the MDL regularization in the DIB objective in \textbf{P3}. A smaller $D_k$, which means higher generalization, corresponds to lower representational complexity, which is precisely what the MDL regularization term $I(S_k, C_{k,-k})$ seeks to minimize. 
  \end{remark}

 \begin{remark}
 The term $\mathcal{M}_{k}$ captures the statistical difficulty of learning from finite and noisy environmental observations in the AgentNet.
A smaller $\mathcal{M}_k$, achieved through more training samples $n$ or lower loss variance $\sigma^2$, directly tightens the generalization bound.
This aligns with the DIB objective in \textbf{P3}: the task-relevance term $I(Y_k; C_{-k,k})$ encourages the emergent protocol to extract semantically stable features, thereby reducing the effective variance $\sigma^2$, while the MDL term $I(S_k; C_{k,-k})$ prevents the protocol from overfitting to spurious environmental noise.
Together with $D_k$, the term $\mathcal{M}_k$ provides a rigorous decomposition of the generalization error into data-driven uncertainty and protocol complexity, jointly justifying the information-theoretic design of the proposed framework.
\end{remark}

\section{Prototype and Experimental Results}
\label{Section_Prototype}


\subsection{Prototype and Experiment Setup}
We developed an emergent communication prototype, consisting of user equipments (UEs), a gNodeB (gNB), and a 5G core (5GC) network, based on an open-source RAN and a softwareized 5G core network platform. 
We consider emergent communication between two types of agents deployed in two layers of a networking system: an application-layer agent and a physical-layer agent. To ensure the results reflect real-world traffic dynamics, the application-layer agent utilizes a publicly available dataset of real-world smartphone traffic\cite{choi2023ml}. This dataset includes data generated from five popular mobile applications, covering diverse services such as live streaming, video conferencing, and mobile gaming. For the physical-layer agent, the real-world traffic is fed into our developed hardware prototype to evaluate the specific physical-layer resource configurations required to support each application's unique requirements. The experiment investigates how these two agents develop an emergent communication protocol to manage network resources. Specifically, the application-layer agent observes and predicts real-time dynamics of the application traffic and learns to send optimized communication signals to the physical-layer agent. In response to the signals received from the application-layer agent, the physical-layer agent dynamically adjusts and allocates hardware resources to ensure the underlying physical-layer resource configurations can support the application's real-time data demands. 

To evaluate the effectiveness of our framework, we compare the proposed solution against a state-of-the-art (SOTA) benchmark, called {\bf EC-SOTA}, proposed in \cite{Lin2021EmergCommAutoEncoder}. This baseline employs a modular approach in which the task-execution and communication-signaling models are trained independently. Specifically, it utilizes a dedicated autoencoder to learn latent representations of application-layer data, which are then transmitted to the physical-layer agent.

\subsection{Experimental Results}


\begin{figure}
    \centering
    \includegraphics[width=0.6\linewidth]{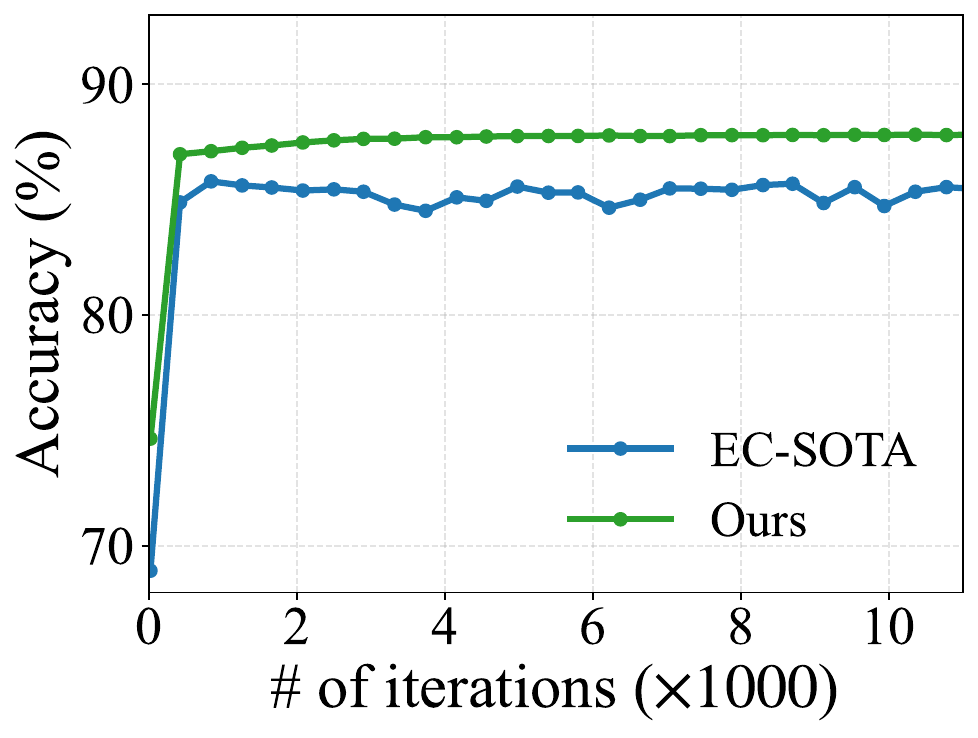}
    \vspace{-0.1in}
    \caption{\footnotesize Application-layer agent's accuracy under different iteration numbers. }
    \label{Fig_AppAgentConvergence}
    \vspace{-0.1in}
\end{figure}


\begin{figure}
    \centering
    \begin{minipage}{0.155\textwidth}
        \centering
        \includegraphics[width=\linewidth]{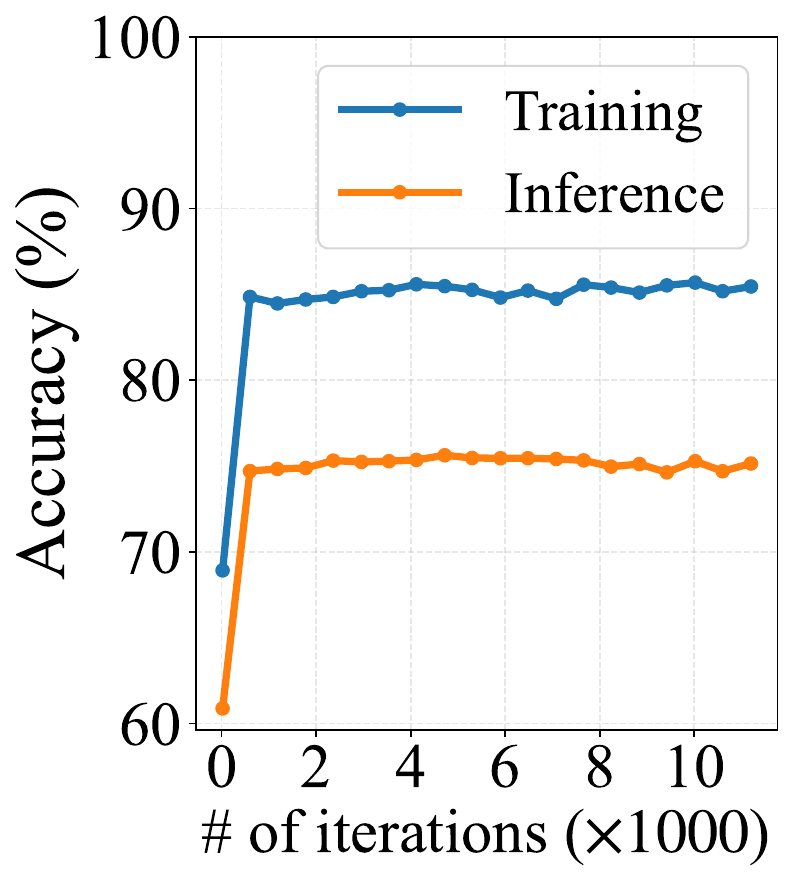} 
        \small (a)
    \end{minipage}
    \hfill
    \begin{minipage}{0.155\textwidth}
        \centering
        \includegraphics[width=\linewidth]{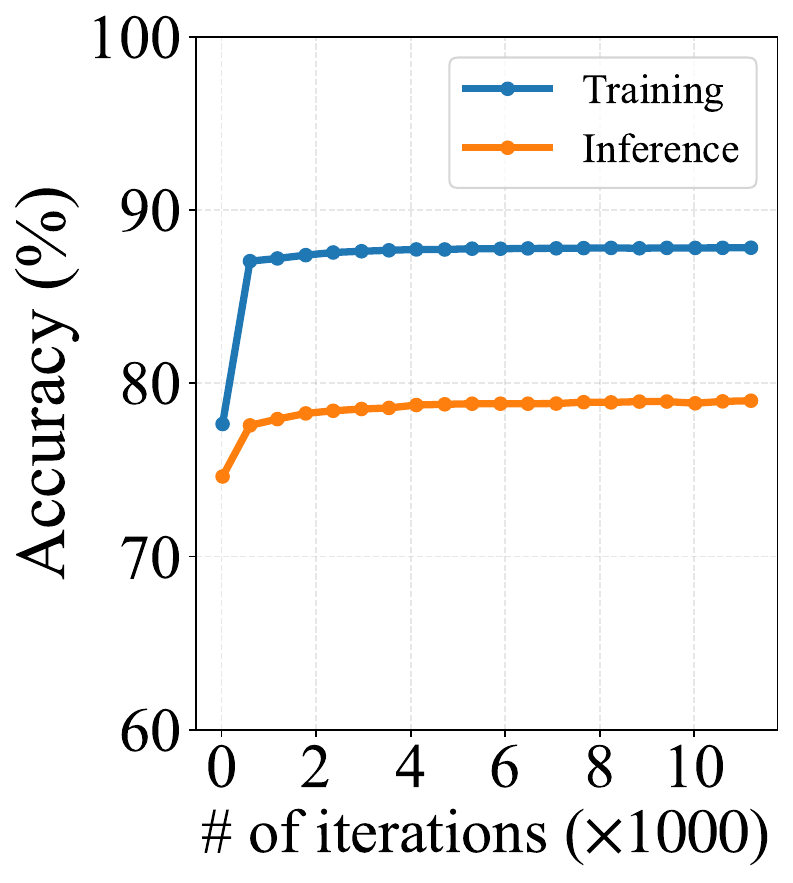} 
        \small (b) 
    \end{minipage}
    \hfill
    \begin{minipage}{0.155\textwidth}
        \centering
        \includegraphics[width=\linewidth]{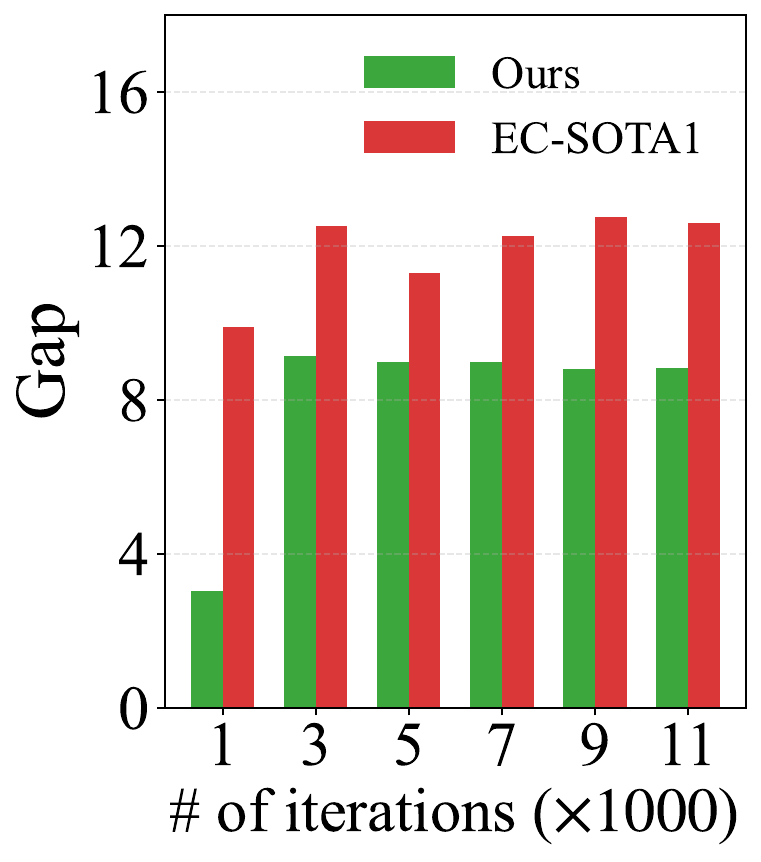} 
        \small (c) 
    \end{minipage}
    \caption{\footnotesize Generalization errors 
    of (a) EC-SOTA and (b) our proposed solution, and (c) comparison of both solutions under different numbers of iterations.}
    \label{Fig_GeneralizationError}
\end{figure}

In Fig. \ref{Fig_AppAgentConvergence}, we present the accuracy of the application-layer agent under different numbers of iterations. We observe that our proposed solution consistently achieves a higher accuracy level than the EC-SOTA solution. 
In Fig. \ref{Fig_GeneralizationError}, we compare the generalization error of our proposed solution with that of the EC-SOTA benchmark. As illustrated in Fig. \ref{Fig_GeneralizationError} (a) and (b), the gap between training and inference accuracy is significantly narrower for our proposed solution across traffic generated by various smartphone applications, including the high-bandwidth and latency-sensitive scenarios such as Live Streaming and Mobile Gaming. Fig. \ref{Fig_GeneralizationError} (c) further highlights the efficiency of our proposed framework, which converges to a significantly lower generalization error at a much faster speed than the EC-SOTA, maintaining a stable error floor as iterations increase. This further justifies the effectiveness of our proposed joint loss function with the multi-agent and multi-task DIB regularizer, which successfully filters out training noise, thereby capturing the underlying semantic features of the application traffic.

\section{Conclusion}
\label{Section_Conclusion}
This paper establishes an information-theory framework for emergent communication within the evolving AgentNet landscape. 
Our proposed framework unifies the task-specific decision-making and emergent communication signaling by 
leveraging the multi-agent and multi-task DIB, which establishes a theoretical foundation for quantifying and optimizing the balance between computational and communication resource efficiency and task-relevance information representation. 
We derive theoretical bounds on the generalization error. 
Validation on a real-world hardware prototype demonstrates that our proposed framework not only outperforms existing heuristic-driven models but also offers a scalable, theoretically sound foundation for the future of fully autonomous AgentNet in 6G and beyond systems.

\section*{Acknowledgment}
This work was supported by the National Natural Science Foundation of China under grant 62525109 and the Mobile Information Network National Science and Technology Key Project under grant 2024ZD1300700.

\bibliographystyle{IEEEtran}
\bibliography{DeepLearningRef}

\end{document}